\documentclass{article}

\usepackage{microtype}
\usepackage{graphicx}
\usepackage{subfigure}
\usepackage{booktabs} 

\usepackage{hyperref}

\usepackage[accepted]{icml2020}

\usepackage{enumerate}
\usepackage{times}
\usepackage[OT1]{fontenc}
\usepackage{amsmath,amssymb}
\usepackage{mathrsfs}
\usepackage{multirow}

\usepackage[T1]{fontenc}    
\usepackage{booktabs}       
\usepackage{amsfonts}       
\usepackage{nicefrac}       
\usepackage{microtype}      

\usepackage{amsmath, amssymb}
\usepackage{algorithm}
\usepackage{algorithmic}
\usepackage{amsthm}
\usepackage[small]{caption}
\usepackage{graphicx}

\usepackage{enumitem}
\usepackage[olditem,oldenum]{paralist}

\DeclareMathOperator{\E}{\mathbb{E}}
\DeclareMathOperator*{\argmin}{arg\,min}
\newtheorem{mytheorem}{Theorem}

\newtheorem{lemma}{Lemma}
\newtheorem{assumption}{Assumption}

\newsavebox\thmbox
\usepackage{xcolor}
\newenvironment{theorem}%
  {\begin{lrbox}{\thmbox}%
   \begin{minipage}{\dimexpr\linewidth-2\fboxsep}
   \begin{mytheorem}}%
  {\end{mytheorem}%
   \end{minipage}%
   \end{lrbox}%
   \begin{trivlist}
     \item[]\colorbox{white}{\usebox\thmbox}
   \end{trivlist}}

\renewcommand{\S}{\mathcal{S}}
\newcommand{\A}{\mathcal{A}}

\newcommand{\calH}{\mathcal{H}}
\newcommand{\calD}{\mathcal{D}}
\newcommand{\calQ}{\mathcal{Q}}
\newcommand{\calV}{\mathcal{V}}
\newcommand{\calP}{\mathcal{P}}
\newcommand{\calG}{\mathcal{G}}
\newcommand{\calC}{\mathcal{C}}

\newcommand{\el}{\mathscr{L}}
\newcommand{\er}{\mathscr{R}}
\newcommand{\var}{\mathbb{V}}
\renewcommand{\P}{\mathbb{P}}

\newcommand{\Real}{\mathbb{R}}
\usepackage{hyperref}       

\hypersetup{
	plainpages=false,
	colorlinks=true,              
	linkcolor=blue,               
	anchorcolor=blue,             
	citecolor=blue,               
	filecolor=blue,               
	pagecolor=blue,               
	urlcolor=blue,                
	pdfview=FitH,                 
	pdfstartview=FitH,            
	pdfpagelayout=SinglePage      
}

\icmltitlerunning{Sharp Analysis of Smoothed Bellman Error Embedding}

\begin{document}

\twocolumn[
\icmltitle{Sharp Analysis of Smoothed Bellman Error Embedding}




\begin{icmlauthorlist}
\icmlauthor{Ahmed Touati}{udem}
\icmlauthor{Pascal Vincent$^{\dagger}$}{udem}

\end{icmlauthorlist}

\icmlaffiliation{udem}{Mila, Universit\'e de Montr\'eal. $^{\dagger}$ Canada CIFAR AI Chair}

\icmlcorrespondingauthor{Ahmed Touati}{ahmed.touati@umontreal.ca}

\icmlkeywords{Machine Learning, ICML}

\vskip 0.3in
]



\printAffiliationsAndNotice{}  

\begin{abstract}
The \textit{Smoothed Bellman Error Embedding} algorithm~\citep{dai2018sbeed}, known as SBEED, was proposed as a provably convergent reinforcement learning algorithm with general nonlinear function approximation. It has been successfully implemented with neural networks and achieved strong empirical results. In this work, we study the theoretical behavior of SBEED in batch-mode reinforcement learning. We prove a near-optimal performance guarantee that depends on the representation power of the used function classes and a tight notion of the distribution shift. Our results improve upon prior guarantees for SBEED in ~\citet{dai2018sbeed} in terms of the dependence on the planning horizon and on the sample size. Our analysis builds on the recent work of ~\citet{Xie2020} which studies a related algorithm MSBO, that could be interpreted as a \textit{non-smooth} counterpart of SBEED.
\end{abstract}

\section{Introduction}
In reinforcement learning (RL), an agent interacts with
an unknown environment and seeks to learn a policy
which maps states to distribution over actions to maximise a long-term numerical reward. Recently, many popular off-policy deep RL algorithms have enjoyed many empirical successes on challenging RL domains such as video games and robotics. Their success can be attributed to their ability to scale gracefully to high dimensional state-action spaces thanks to their use of modern high-capacity function approximators such as neural networks. Most of these algorithms have their roots in Approximate Dynamic Programming (ADP) methods~\citep{bertsekas1995dynamic}, which are standard approaches to tackle decision problems with large state space by making successive calls to a supervised learning algorithm. For example, Deep Q-Network (DQN)~\citep{mnih2015human} can be related to Approximate Value Iteration while Soft Actor-Critic (SAC)~\citep{haarnoja2018soft} can be related to Approximate Policy Iteration.

Unfortunately, it is well known that such off-policy methods, when combined with function approximators, fail frequently to converge to a solution and can be even divergent~\citep{baird1995residual, tsitsiklis1997analysis}. Stable approaches have been an active area
of investigation. For example, restrictive function classes such as averagers~\citep{gordon1995stable} or smoothing kernel~\citep{ormoneit2002kernel} were shown to lead to stable learning. Gradient-based temporal difference approaches have been proposed to derive convergent algorithms with linear function approximators but only for off-policy evaluation~\citep{sutton2008convergent, touati2018convergent}.

Recently, \textit{Smoothed Bellman Error Embedding} (SBEED) was introduced in~\citet{dai2018sbeed} as the first provably convergent algorithm with general function approximators.~\citet{dai2018sbeed} leverage Nesterov's smoothing technique and the convex-conjugate trick to derive a primal-dual optimization problem. The algorithm learns the optimal value function and the optimal policy in the primal, and the Bellman residual in the dual.

In this work, we study the theoretical behavior of SBEED in batch-mode reinforcement learning where the algorithm has only access to a fixed dataset of transitions. We prove a near-optimal performance guarantee that depends on the representation power of the function classes we use and a tight notion of the distribution shift. Our results improve upon prior guarantee of SBEED, presented in the original paper~\citet{dai2018sbeed}, in terms of the dependence on the planning horizon and on the sample size. In particular, we show that SBEED enjoys linear dependence on horizon, which is the best that we can hope~\citep{scherrer2012use}, and that the statistical error decreases in the rate $1/\sqrt{n}$ instead of $1/\sqrt[4]{n}$ provided that function classes are rich enough in a sense that we will specify. Our analysis builds on the recent work of~\citet{Xie2020} that studies a related algorithm MSBO, which could be interpreted as a \textit{non-smooth} counterpart of SBEED. However, both algorithms differ in several aspects: SBEED learns jointly the optimal policy and the optimal value function while MSBO learns the optimal $Q$-value function and considers only policies that are greedy with respect to it. Moreover, even as the smoothing parameter goes to zero, SBEED's learning objective does not recover MSBO's objective.

\section{Preliminaries and Setting}
\subsection{Markov Decision Processes} 

We consider a discounted Markov Decision Process (MDP) defined by a tuple $(\S, \A, \gamma, P, R, d_0)$ with state space $\S$, action space $\A$, discount factor $\gamma \in [0, 1)$, transition probabilities $P \in  \Delta(\S)^{\S \times \A}$ mapping state-action pairs to distributions over next states ($\Delta(\cdot)$ denotes the probability simplex), and reward function $R \in  [0, R_{\max} ]^{\S \times \A}$. $d_0 \in \Delta(\S)$ is the initial state distribution. For the sake of clarity, we assume the state and action spaces are finite whose cardinality can be arbitrarily large, but our analysis can be extended to the countable or continuous case. We denote by $\pi(a \mid s)$ the probability of choosing action $a$ in state $s$ under the policy $\pi \in  \Delta(\A)^{\S}$. The performance of a policy $\pi$ represents the expected sum of discounted rewards: $J(\pi) \triangleq \E[ \sum_{t=0}^{\infty} \gamma^t r_t \mid s_0 \sim d_0, \pi]$ where the expectation is taken over trajectories $\{s_0, a_0, r_0, s_1, a_1, r_1, \ldots \}$ induced by the policy in the MDP such that $s_0 \sim d_0(\cdot), a_t \sim \pi(\cdot \mid s_t), r_t = R(s_t, a_t) \text{ and } s_{t+1} \sim P( \cdot \mid s, a)$. Moreover, we define value function $V^\pi(s) \triangleq \E [ \sum_{t=0}^{\infty} \gamma^t r_t \mid s_0=s, \pi]$ and Q-value function $Q^\pi(s, a) \triangleq \E [ \sum_{t=0}^{\infty} \gamma^t r_t \mid (s_0, a_0)=(s, a), \pi]$. These functions take value in $[0, V_{\max}]$ where $V_{\max} \triangleq R_{\max} /(1-\gamma)$.

We define the discounted state occupancy
measure $d^\pi$ induced by a policy $\pi$ as 
\begin{equation*}
    d^\pi(s) \triangleq (1-\gamma) \sum_{t=0}^\infty \gamma^t d^\pi_t(s),
\end{equation*}
where $d^\pi_t(s) \triangleq \text{Pr}(s_t =s \mid s_0 \sim d_0, \pi)$ is the probability that $s_t=s$ after we execute $\pi$ for $t$ steps, starting from initial state $s_0 \sim d_0$. By definition, $d^\pi_0 = d_0$. Similarly, we define $d^\pi(s, a) = d^\pi(s) \cdot \pi(a \mid s)$.
\paragraph{Entropy Regularized MDP:} The idea of entropy regularization has also been widely used in the RL literature. In entropy regularized MDP, also known as soft MDP, we aim at finding the policy $\pi_\lambda^\star \in \Delta(\A)^{\S} $ that maximizes the following objective:
\begin{align}
J_\lambda(\pi) & \triangleq \E \left [ \sum_{t=0}^{\infty} \gamma^t \left (r_t - \lambda \ln \pi(a_t \mid s_t)\right) \mid s_0 \sim d_0, \pi \right ] \nonumber\\
& = J(\pi) + \lambda \E \left [ \sum_{t=0}^\infty \gamma^t \calH ( \pi(\cdot \mid s_t)) \mid s_0 \sim d_0, \pi \right], \nonumber \label{eq: softMDP}
\end{align}
where $\calH(\pi(\cdot \mid s_t)) = -\E_{a \sim \pi(\cdot \mid s)} [\ln\pi(a \mid s)]$ is the Shannon entropy function and $\lambda$ is a regularization parameter.

\subsection{Batch Reinforcement Learning}
We are concerned with the batch RL setting where an agent does not have the ability to interact with the environment, but is instead provided with a batch dataset $\calD = \{ s_i, a_i, r_i, s'_i \}_{i \in [n]}$ such that for every $i \in [n]$,  $(s_i, a_i)$ is an i.i.d sample generated from a data distribution $\mu \in \Delta(\S \times \A)$, $r_i = R(s_i, a_i)$ and $s'_i \sim  P(\cdot \mid s_i, a_i)$.

A typical batch learning algorithm requires the access to a function class $\calQ \subset [0, V_{\max}]^{\S \times \A}$ and aims at computing a near-optimal policy from the data by approximating the optimal action-value function $Q^\star$ with some element $Q$ of $\calQ$ and then outputing the greedy policy with respect to $Q$. Different algorithms suppose access to different function classes. As a further simplification, we assume that all function classes have finite but exponentially large cardinalities. 

\section{Smoothed Bellman Error Embedding}
In this section, we describe the SBEED algorithm and provide insights about its design that will be useful for our subsequent analysis. The next lemma restates Proposition 2 in~\citet{dai2018sbeed} that characterizes the optimal value-function and the optimal policy of the soft MDP.
\begin{lemma}[Temporal consistency Proposition 3 in \citet{dai2018sbeed}] \label{lemma: consistency} The optimal value function $V^\star_\lambda$ and the optimal policy $\pi^\star_\lambda$ of the soft MDP are the unique $(V, \pi)$ pair that satisfies the following equality for all $(s, a) \in \S \times \A:$
\begin{equation*}
    V(s) = R(s, a) + \gamma (\P V)(s, a) - \lambda \ln \pi(a \mid s)
\end{equation*}
where $(\P V)(s, a) \triangleq \E_{s' \sim P(\cdot \mid s, a)}[V(s')]$.
\end{lemma}
Let $\calC_\lambda^\pi$ denote the \textit{consistency} operator defined for any $(s, a) \in \S \times \A$ by $(\calC_\lambda^\pi V)(s, a) = R(s, a) + \gamma (\P V)(s, a) - \lambda \ln \pi(a \mid s)$.
A natural objective function inspired by Lemma~\ref{lemma: consistency} would be:
\begin{equation} \label{eq: ultimate loss}
    \min_{V \in  \calV, \pi \in \calP} \| V - \calC_\lambda^\pi V\|^2_{2, \mu},
\end{equation}
where $\forall f \in \Real^{\S \times \A},  \| f \|^2_{2, \mu} \triangleq \E_{(s, a) \sim \mu} [f(s, a)^2]$ is the $\mu$-weighted 2-norm. $\calV \subset [0, V_{\lambda, \max}]^{\S \times \A}$, with $V_{\lambda, \max} \triangleq \frac{R_{\max} + \lambda \ln|\A|}{1-\gamma}$, is the class of candidate value functions and $\calP \subset \{ \pi \in \Delta(\A)^{\S}, \| \ln \pi \|_{\infty} \leq V_{\lambda, \max} / \lambda \}$\footnote{According to Lemma~\ref{lemma: consistency}, $0 \geq \lambda \ln \pi^\star_{\lambda}(a \mid s) = R(s, a) + \gamma (\P V^\star_\lambda)(s, a) - V^\star_\lambda(s) \geq - V^\star_\lambda(s) \geq - V_{\lambda, \max} \Rightarrow \| \ln \pi^\star_\lambda\|_\infty \leq V_{\lambda, \max} / \lambda$}, is the class of candidate policies. 
To solve the minimization problem~\eqref{eq: ultimate loss}, one may try to minimize the empirical objective estimated purely from data:
\begin{align*}
\el_{\calD}&(V; V, \pi) \triangleq \\ & \frac{1}{n} \sum_{i=1}^n \Big( V(s_i) - r_i - \gamma V(s'_i) + \lambda \ln \pi(a_i \mid s_i) \Big)^2 
\end{align*}
Due to the inner conditional expectation in~\eqref{eq: ultimate loss}, the expectation of $\el_{\calD}(V; V, \pi)$ over the draw of dataset $\calD$ is different from the original objective $\| V - \calC_\lambda^\pi V\|^2_{2, \mu}$. In particular.
\begin{align} \label{eq: conditional_variance}
    \E[\el_{\calD}&(V; V, \pi)] = \\
    & \| V - \calC_\lambda^\pi V\|^2_{2, \mu} + \gamma^2 \E_{(s,a) \sim \mu}[ \var_{s' \sim P(\cdot \mid s, a)}[V(s')]] \nonumber
\end{align}
To address this issue, also known as the double sampling issue~\citep{baird1995residual}, \citet{dai2018sbeed} use the Frenchel dual trick as well as an appropriate change of variable and derive the following minimax SBEED objective:
\begin{equation} \label{eq: minimax}
    \min_{V \in  \calV, \pi \in \calP} \max_{g \in \calG}\el_{\calD}(V; V, \pi) - \er_{\calD}(g; V, \pi),
\end{equation}
where $\calG \subset [0, 2 V_{\lambda, \max}]^{\S \times \A}$ is a helper function class and
\begin{align*}
    \er_{\calD}(g;& V, \pi) \triangleq \\ & \frac{1}{n} \sum_{i=1}^n \Big( g(s_i, a_i) - r_i - \gamma V(s'_i) + \lambda \ln \pi(a_i \mid s_i)\Big)^2.
\end{align*}

To understand the intuition behind the SBEED objective~\eqref{eq: minimax}, note that if $\calG$ is rich enough, in the regime of infinite amount of data, the minimizer of the regression problem $\min_{g \in \calG} \er_\calD(g; V, \pi)$ converges to $\calC_\lambda^\pi V$, which is the Bayes optimal regressor, and the minimum converges to the conditional variance $\gamma^2 \E_{(s,a) \sim \mu}[ \var_{s' \sim P(\cdot \mid s, a)}[V(s')]]$, which is the optimal Bayes error. This allows the cancellation of the extra conditional variance term in Equation~\eqref{eq: conditional_variance}. Therefore, $\max_{g \in \calG}\el_{\calD}(V; V, \pi) - \er_{\calD}(g; V, \pi)$ is a consistent estimate of $\| V - \calC_\lambda^\pi V\|^2_{2, \mu}$ as long as $\calG$ is rich enough.

Note that the only difference between $\el_\calD$ and $\er_\calD$ is that the former takes single-variable function ($V \in \Real^{\S}$) as first argument while the latter takes two-variable function ($g \in \Real^{\S \times \A}$) as the first argument. 

\section{Analysis}
In this section, we provide a near-optimal performance guarantee for the SBEED algorithm. In order to state our main results, we need to introduce a few key assumptions. The first characterizes the distribution shift, more precisely, the mismatch between the training distribution $\mu$ and the discounted occupancy measure induced by any policy $\pi \in \calP \cup \{ \pi^\star_\lambda \}$. 
\begin{assumption}[Concentrability coefficient] \label{assump: concentrability} we assume that
$C_2 \triangleq \max_{\pi \in \calP \cup \{ \pi^\star_\lambda \}}\Big \| \frac{d^\pi}{\mu}\Big \|^2_{2, \mu} < \infty$.
\end{assumption}
$C_2$ uses the $\mu$-weighted square of the marginalized importance weights $d^\pi / \mu$ and it is one of the simplest versions of concentrability coefficients considered in the literature~\citep{Munos03, antos2008learning, scherrer2014approximate}. In spite of its simple form, $C_2$ could be tighter than more involved concentrability coefficients in some cases~\citep{Xie2020}.

Now, we introduce the assumptions that characterize the representation power of the function classes. The next assumption measures the capacity of the policies and value spaces to represent the optimal policy and the optimal value function of the soft MDP. 
\begin{assumption}[Approximate realizability] \label{assump: realizability}
$\epsilon_{\calV, \calP} \triangleq \min_{V \in \calV, \pi \in \calP} \| V - \calC_\lambda^\pi V\|^2_{2, \mu} < \infty$
\end{assumption}
According to Lemma~\ref{lemma: consistency}, $\calP$ and $\calV$ realize $\pi^\star_\lambda$ and $V^\star_\lambda$ ($\pi^\star_\lambda \in \calP$ and $V^\star_\lambda \in \calV$) implies that $\epsilon_{\calV, \calP} = 0$. Therefore, $\epsilon_{\calV, \calP}$ measures the violation of the realizability of $\calP$ and $\calV$.

\begin{assumption}[Approximate realizability of the helper class] \label{assump: realizability of G}
$\epsilon_{\calG, \calV, \calP} \triangleq \max_{V \in \calV, \pi \in \calP} \min_{g \in \calG} \|g - C^\pi_\lambda V\|^2_{2, \mu} < \infty$.
\end{assumption}
When $\calG$ realizes the optimal Bayes regressor $C^\pi_\lambda V$ for any $\pi \in \calP$ and $V \in \calV$,  $\epsilon_{\calG, \calV, \calP} = 0$. Therefore, the latter assumption measures the violation of $C^\pi_\lambda V \in \calG$ for the worst-case $\pi \in \calP$ and $V \in \calV$.

Our analysis starts by stating a useful telescoping Lemma.
\begin{lemma}[Telescoping Lemma]\label{lemma: telescoping}
    For any $V \in \Real^{\S \times \A}$ and $\pi \in \Delta(\A)^{\S}$:
    \begin{align*}
        \E_{s \sim d_0}[V(s)] - J_\lambda(\pi)
        &= \frac{\E_{(s, a) \sim d^\pi}[ V(s) - (\calC_\lambda^\pi V)(s, a)]}{1-\gamma}.
    \end{align*}
\end{lemma}
Lemma~\ref{lemma: telescoping} is an important first step to prove a linear dependence on the planning horizon $1/(1-\gamma)$ of SBEED unlike standard iterative methods, such as Fitted Q-iteration, that incur quadratic dependence on the horizon $1/(1-\gamma)^2$. A similar lemma was proved in~\citet{Xie2020} for Q-value functions of the unregularized MDP.  

Let $\hat{V}$ and $\hat{\pi}$ denote the output of the SBEED algorithm:
\begin{align*}
     \hat{V}, \hat{\pi} \triangleq \argmin_{V \in  \calV, \pi \in \calP} \max_{g \in \calG}\el_{\calD}(V; V, \pi) - \er_{\calD}(g; V, \pi).
\end{align*}
With the telescoping lemma~\ref{lemma: telescoping} and the definition of concentrability coefficient $C_2$ , we can relate the suboptimality of the learnt policy $\hat{\pi}$ with the minimization objective Equation~\eqref{eq: ultimate loss}.
\begin{lemma}[Suboptimality] \label{lemma: suboptimality} 
The performance difference between the optimal policy and the output policy, is given by
\begin{equation} \label{eq: suboptimality}
    J(\pi^\star) - J(\hat{\pi}) \leq \frac{\lambda \ln |\A|}{1-\gamma} + \frac{2 \sqrt{C_2}}{1-\gamma} \| \hat{V}- \calC_\lambda^{\hat{\pi}} \hat{V}\|_{2, \mu}.
\end{equation}
\end{lemma}
The first term of the left-hand-side of Equation~\eqref{eq: suboptimality} is the bias due to the entropy regularization. To be able to establish our performance guarantee we need to relate the second term $ \| \hat{V}- \calC_\lambda^{\hat{\pi}} \hat{V}\|_{2, \mu}$ to the empirical loss $\el_{\calD}(\hat{V}; \hat{V}, \hat{\pi}) - \er_{\calD}(\hat{g}_{\hat{V}, \hat{\pi}}; \hat{V}, \hat{\pi})$ in the minimax objective in Equation~\eqref{eq: minimax} that SBEED solves where we define $\hat{g}_{V, \pi} \triangleq \argmin_{g \in \calG} \er_{\calD}(g; V, \pi)$ for any $V, \pi$.
The former is a population loss while the latter is an empirical loss estimated from the dataset $\calD$ and involves a helper function $\hat{g}_{\hat{V}, \hat{\pi}}$. 
We first drop the dependence on the helper function $\hat{g}_{\hat{V}, \hat{\pi}}$ by bounding the deviation between $\er_{\calD}(\hat{g}_{V, \pi}; V, \pi)$ and $\er_{\calD}(\calC_\lambda^\pi V; V, \pi)$ uniformly over $\calV$, $\calP$ and $\calG$. We get informally:
\begin{align*}
    \el_{\calD}(\hat{V}; \hat{V}, \hat{\pi}) - &  \er_{\calD}(\hat{g}_{\hat{V}, \hat{\pi}}; \hat{V}, \hat{\pi}) \approx \\
    & \el_{\calD}(\hat{V}; \hat{V}, \hat{\pi}) -  \er_{\calD}(\calC_\lambda^{\hat{\pi}} \hat{V}; \hat{V}, \hat{\pi})
\end{align*} 
We can finally obtain the desired result by exploiting the fact that $\el_{\calD}(\hat{V}; \hat{V}, \hat{\pi}) -  \er_{\calD}(\calC_\lambda^{\hat{\pi}} \hat{V}; \hat{V}, \hat{\pi})$ is equal in expectation over draws of the dataset $\calD$ to the quantity of interest $\| \hat{V}- \calC_\lambda^{\hat{\pi}} \hat{V}\|^2_{2, \mu}$. Thorough treatment of each step involves using concentration of measures as well as dealing with function approximation errors. In particular, we use Bernstein inequality in order to get faster rate, similarly to what was used for Fitted Q-iteration and MSBO analysis in ~\citet{chen2019information}. Detailed proofs are provided in the supplemental.  
We now state our performance guarantee.
\begin{theorem}[Performance guarantee of SBEED] \label{theorem: performance}
With probability at least $1-\delta$
\begin{align*}
    J(\pi^\star) & - J(\hat{\pi})  \leq \frac{\lambda \ln |\A|}{1-\gamma} + \frac{\sqrt{C_2}}{1-\gamma} {\cal O} 
    \left( \sqrt{\epsilon_{\calG, \calV, \calP}} + \sqrt{\epsilon_{\calV, \calP}} \right) \\
    & + \frac{\sqrt{C_2}}{1-\gamma} {\cal O} \left (
    \sqrt[4]{\frac{V_{\lambda, \max}^2  \ln \frac{|\calV| |\calP| |\calG|}{\delta}}{n} (\epsilon_{\calG, \calV, \calP} + \epsilon_{\calV, \calP})} \right) \\
    &  + \frac{\sqrt{C_2}}{1-\gamma} {\cal O} \left( \sqrt{\frac{V_{\lambda, \max}^2  \ln \frac{|\calV| |\calP| |\calG|}{\delta}}{n}} \right)
\end{align*}
\end{theorem}
As an immediate consequence of Theorem~\ref{theorem: performance}, we provide a finite sample complexity of SBEED in the case of full realizability i.e $\epsilon_{\calG, \calV, \calP} = \epsilon_{\calV, \calP} = 0$.

\begin{theorem}[Sample complexity of SBEED] $\forall \epsilon, \delta >0$, if $\lambda \leq \frac{(1-\gamma)\epsilon}{2 \ln |\A|}$ and $\epsilon_{\calG, \calV, \calP} = \epsilon_{\calV, \calP} = 0$, then with probability at least $1- \delta$, we have $J(\pi^\star) - J(\hat{\pi}) \leq \epsilon$, if 
\begin{equation*}
    n = \mathcal{O}\left(\frac{C_2 \cdot V_{\lambda, \max}^2 \ln \frac{|\calV| |\calP| |\calG|}{\delta}}{ \epsilon^2 (1-\gamma)^2}\right)
\end{equation*}

\end{theorem}

\paragraph{Comparison with prior analysis of SBEED:} Our results improve over the original analysis of SBEED in~\citet{dai2018sbeed} in many aspects. First, in terms of guarantee,~\citet{dai2018sbeed} provide a bound on the $\mu$-weighted distance between the learnt value function and the optimal one $\|\hat{V} - V^\star\|_{2, \mu}$, which does not really capture the performance of the algorithm. In fact, the learnt value function does not necessarily correspond to the value of any policy. It is rather used to learn a policy $\hat{\pi}$ which will be executed in the MDP. Therefore, the quantity of interest that we should be looking at instead, as we did in this work, is the difference between the optimal performance and the performance of $\hat{\pi}$: $J^\star - J(\hat{\pi})$. Secondly, in terms of statistical error, we obtain a faster rate of $1/\sqrt{n}$ in the fully realizable case thanks to the use of Bernstein inequality~\citep{cesa2006prediction} while ~\citet{dai2018sbeed} prove a slower rate of  $1/\sqrt[4]{n}$. When the realizability holds only approximately, the slow rate of our bound in Theorem~\ref{theorem: performance} $\sqrt[4]{\frac{\epsilon_{\calG, \calV, \calP} + \epsilon_{\calV, \calP}}{n}}$ can be made smaller than the $1/\sqrt[4]{n}$ of~\citet{dai2018sbeed} by decreasing the approximation error. 

\paragraph{Comparison with MSBO:}~\citet{Xie2020} studies a related algorithm called MSBO, which solves the following minimax objective:
\begin{align*}
    \max_{Q \in \mathcal{Q}}&\min_{f \in \mathcal{F}} l_{\calD}(Q; Q) - l_{\calD}(f; Q), \text{ where}\\
    l_\calD (Q;Q') \triangleq &\frac{1}{n} \sum_{i=1}^n \Big( Q(s_i, a_i) - r - \gamma \max_{a'} Q'(s'_i, a')\Big)^2.
\end{align*}
SBEED can be seen as a \textit{smooth} couterpart to MSBO. Except for the bias due to entropy regularization, our performance bound shares a similar structure to the bound obtained by~\citet{Xie2020} for MSBO, but with our own definition of the concentrability coefficient and approximation errors that suit our algorithm of interest. However, the two algorithms differ in several manners. SBEED learns jointly both the policy and value function while MSBO learns the $Q$-value function and considers only policies that are greedy with respect to it. If we set $\lambda=0$ in the SBEED objective, we don't recover the MSBO objective, which means that MSBO is not a special case of SBEED. In fact, when $\lambda=0$, SBEED will learn the value function $V^\pi$ of the behavior policy that generates the data $\pi (a \mid s) = \frac{\mu(s, a)}{\sum_{a'} \mu(s, a')}$. Finally, it is established that SBEED, when implemented with a differentiable function approximator, would converge locally while there is no practical convergent instantiation of MSBO, as far as we know.

\section{Conclusion and Future Work}
We establish a performance guarantee of the SBEED algorithm that depends only linearly on the planning horizon and enjoys improved statistical rate in the realizable case. Our bound matches the bound of MSBO, a \textit{non-smooth} counterpart of SBEED, which suggests that there is no clear benefit of the entropy regularization. 
As future work, we would like to look at regularized versions of Fitted Policy Iteration or Fitted Q Iteration that have weaker guarantee than SBEED and investigate whether the regularization would play a more significant role to improve their performance.  

\section{Acknowledgements}
We would like to thank Harsh Satija for his helpful feedback on the paper.

\bibliographystyle{apalike}
\bibliography{lib}

\newpage
\appendix
\onecolumn

\section{Outline}
The appendix of this paper is organized as follows:
\begin{compactenum}[\hspace{0pt} 1.]
    \setlength{\itemsep}{2pt}
    \item Appendix \ref{sec:notation} provides a table of notation for easy reference.
    \item Appendix \ref{sec: proof of telecoping lemma} provides the proof of Lemma~\ref{lemma: telescoping} .
    \item Appendix \ref{sec: suboptimality} provides the proof of Lemma~\ref{lemma: suboptimality}.
    \item Appendix~\ref{sec: theorem} provides the proof of Theorem~\ref{theorem: performance}.
    \item Appendix \ref{sec: technical lemmas} provides some technical lemmas.
\end{compactenum}

\section{Notations} \label{sec:notation}
We provide this table for easy reference. Notation will also be defined as it is introduced.

\begin{table*}[h] 
\begin{center}
\caption{Notation table}
\begin{tabular}{l l l}
 \hline
  $J(\pi)$ & $\triangleq$ & $\E[ \sum_{t=0}^{\infty} \gamma^t r_t \mid s_0 \sim d_0, \pi]$ policy performance in the unregularized MDP \\
   $\pi^\star, V^\star, J^\star$ & $\triangleq$ & optimal policy, value function and performance of the unregularized MDP \\
  $J_\lambda(\pi)$ & $\triangleq$ & $\E \left [ \sum_{t=0}^{\infty} \gamma^t \left (r_t - \lambda \ln \pi(a_t \mid s_t)\right) \mid s_0 \sim d_0, \pi \right ]$ policy performance in the soft MDP \\
  $\pi_\lambda^\star, V_\lambda^\star, J_\lambda^\star$ & $\triangleq$ & optimal policy, value function and performance of the soft MDP \\
  $V_{\lambda, \max}$ & $\triangleq$ & $(R_{\max} + \lambda \ln|\A|) / (1-\gamma)$ maximum value taken by the value function of the soft MDP \\
  $\calC_\lambda^\pi$ & $\triangleq$ & Consistency operator defined by $(\calC_\lambda^\pi V)(s, a) = R(s, a) + \gamma (\P V)(s, a) - \lambda \ln \pi(a \mid s)$.\\
  $\calV$ & $\triangleq$ & class of candidate value function  $\subset [0, V_{\lambda, \max}]^{\S \times \A}$ \\
  $\calP$ & $\triangleq$ & class of candidate policies $\subset \{ \pi \in \Delta(\A)^{\S}, \| \ln \pi \|_{\infty} \leq V_{\lambda, \max} / \lambda \}$ \\
  $\calG$ & $\triangleq$ & class of helper functions $\subset [0, 2 V_{\lambda, \max}]^{\S \times \A}$\\
  $\| \cdot \|_{2, \mu}$ & $\triangleq$ & the $\mu$-weighted 2-norm $\forall f \in \Real^{\S \times \A},  \| f \|^2_{2, \mu} \triangleq \E_{(s, a) \sim \mu} [f(s, a)^2]$  \\
  $\el_{\calD}(V; V, \pi)$ &  $\triangleq$ & $ \frac{1}{n} \sum_{i=1}^n \Big( V(s_i) - r_i - \gamma V(s'_i) + \lambda \ln \pi(a_i \mid s_i) \Big)^2 $ \\
  $\er_{\calD}(g; V, \pi)$ & $\triangleq$ & $ \frac{1}{n} \sum_{i=1}^n \Big( g(s_i, a_i) - r_i - \gamma V(s'_i) + \lambda \ln \pi(a_i \mid s_i)\Big)^2$ \\
  $\hat{V}, \hat{\pi}$ & $\triangleq$ &  $\argmin_{V \in  \calV, \pi \in \calP} \max_{g \in \calG}\el_{\calD}(V; V, \pi) - \er_{\calD}(g; V, \pi)$ the output of SBEED algorithm\\
  $\hat{g}_{V, \pi}$ & $\triangleq$ & $\argmin_{g \in \calG} \er_{\calD}(g; V, \pi)$ \\
  $\bar{V}, \bar{\pi}$ &  $\triangleq$ & $\argmin_{V \in \calV, \pi \calP} \| V- \calC_\lambda^{\pi} V\|_{2, \mu}^2$ the best solution in $\calV$ and $\calP$ \\
  $\bar{g}_{V, \pi}$ & $\triangleq$ &  $\argmin_{g \in \calG} \|g - \calC_\lambda^\pi V \|^2_{2, \mu}$ the best solution in $\calG$ \\
  $C_2$ & $\triangleq$ & $\max_{\pi \in \calP \cup \{ \pi^\star_\lambda \}}\Big \| \frac{d^\pi}{\mu}\Big \|^2_{2, \mu}$ concentrability coefficient\\
  $\epsilon_{\calV, \calP}$ & $\triangleq$ & $\min_{V \in \calV, \pi \in \calP} \| V - \calC_\lambda^\pi V\|^2_{2, \mu}$ \\
  $\epsilon_{\calG, \calV, \calP}$ & $\triangleq$ & $\max_{V \in \calV, \pi \in \calP} \min_{g \in \calG} \|g - C^\pi_\lambda V\|^2_{2, \mu}$ \\
  $\imath$ &  $\triangleq$ & $V_{\lambda, \max}^2  \ln \frac{|\calV| |\calP| |\calG|}{\delta'}$ \\
  $\jmath$ &  $\triangleq$ & $V_{\lambda, \max}^2  \ln \frac{|\calV| |\calP|}{\delta'}$\\
  \end{tabular}

\end{center}
\end{table*}

\section{Proof of Lemma~\ref{lemma: telescoping}} \label{sec: proof of telecoping lemma}
\begin{proof} We have
\begin{align*}
    \E_{(s, a, r, s') \sim d^\pi}& [V(s) - \gamma V(s')] \\
    & = 
    (1 - \gamma) \left ( \sum_{s \in \S}\sum_{t=0}^\infty \gamma^t d^\pi_t(s) V(s) - \sum_{s, a, s' \in \S \times \A \times \S}\sum_{t=0}^\infty \gamma^{t+1} d^\pi_t(s) \pi(a \mid s) P( s' \mid s, a) V(s')   \right) \\
    & = (1 - \gamma) \left ( \sum_{s \in \S}\sum_{t=0}^\infty \gamma^t d^\pi_t(s) V(s) - \sum_{s' \in \S}\sum_{t=0}^\infty \gamma^{t+1} d^\pi_{t+1}(s') V(s')   \right) \tag{by marginalizing over $(s, a) \in \S \times \A$}\\
    & = (1- \gamma) \sum_{s \in \S} d_0(s) V(s) = (1-\gamma) \E_{s \sim d_0}[V(s)]
\end{align*}
We obtain the desired result by noticing that $J_\lambda(\pi) = \frac{\E_{(s, a) \sim d^\pi}[ R(s, a) - \lambda \ln \pi(a \mid s)]}{1-\gamma}$
\end{proof}

\section{Proof of Lemma~\ref{lemma: suboptimality}} \label{sec: suboptimality}
\begin{proof} We start by bounding the performace suboptimality in the soft MDP. 
\begin{align*}
    J_\lambda(\pi^\star_\lambda) - J_\lambda(\hat{\pi}) & = J_\lambda(\pi^\star_\lambda) - \E_{s \sim d_0}[\hat{V}(s)] + \E_{s \sim d_0}[\hat{V}(s)] - J_\lambda(\hat{\pi}) \\
    & = \frac{1}{1-\gamma} \Big( - \E_{(s, a, r, s') \sim d^{\pi^\star_\lambda}} [\hat{V}(s) - r - \gamma \hat{V}(s') + \lambda \ln \pi^\star_\lambda(a \mid s)] \\
    & \quad + \E_{(s, a, r, s') \sim d^{\hat{\pi}}} [\hat{V}(s) - r - \gamma \hat{V}(s') + \lambda \ln \hat{\pi}(a \mid s)]\Big) \tag{apply Lemma~\ref{lemma: telescoping}}\\
    & = \frac{1}{1-\gamma} \Big( - \E_{(s, a, r, s') \sim d^{\pi^\star_\lambda}} [\hat{V}(s) - r - \gamma \hat{V}(s') + \lambda \ln \hat{\pi}(a \mid s)] \\
    & \quad - \lambda \E_{(s, a) \sim d^{\pi^\star_\lambda}}\left[\ln \frac{\pi^\star_\lambda(a \mid s)}{\hat{\pi}(a \mid s)} \right] + \E_{(s, a, r, s') \sim d^{\hat{\pi}}} [\hat{V}(s) - r - \gamma \hat{V}(s') + \lambda \ln \hat{\pi}(a \mid s)]\Big) \\
    & = \frac{1}{1-\gamma} \Big( \E_{(s, a) \sim d^{\pi^\star_\lambda}}\left [ (\calC_\lambda^{\hat{\pi}} \hat{V})(s, a) - \hat{V}(s) \right] - \lambda \E_{s \sim d^{\pi^\star_\lambda}}\left[ D_{\text{KL}}\left(\pi^{\star}_\lambda(\cdot \mid s) \| \hat{\pi}(\cdot \mid s)\right)\right] \\
    & \quad + \E_{(s, a) \sim d^{\hat{\pi}}}\left [ \hat{V}(s) - (\calC_\lambda^{\hat{\pi}} \hat{V})(s, a) \right] \Big) \\
    & \leq \frac{1}{1-\gamma} \Big( \E_{(s, a) \sim d^{\pi^\star_\lambda}}\left [ (\calC_\lambda^{\hat{\pi}} \hat{V})(s, a) - \hat{V}(s) \right] + \E_{(s, a) \sim d^{\hat{\pi}}}\left [ \hat{V}(s) - (\calC_\lambda^{\hat{\pi}} \hat{V})(s, a) \right] \Big) \tag{$D_{\text{KL}}(p \| q) \geq 0$}\\
    & = \frac{1}{1-\gamma} \Big( \E_{(s, a) \sim \mu}\left [ \frac{d^{\pi^\star_\lambda}(s, a)}{\mu(s, a)}\left( (\calC_\lambda^{\hat{\pi}} \hat{V})(s, a) - \hat{V}(s)\right) \right] + \E_{(s, a) \sim \mu}\left [ \frac{d^{\hat{\pi}}(s, a)}{\mu(s, a)} \left(\hat{V}(s) - (\calC_\lambda^{\hat{\pi}} \hat{V})(s, a)\right) \right] \Big) \\
    & \leq \frac{1}{1-\gamma} \Big ( \Big \| \frac{d^{\pi^\star_\lambda}}{\mu} \Big \|_{2, \mu} \| \hat{V}- \calC_\lambda^{\hat{\pi}} \hat{V}\|_{2, \mu} + \Big \| \frac{d^{\hat{\pi}}}{\mu} \Big \|_{2, \mu} \| \hat{V}- \calC_\lambda^{\hat{\pi}} \hat{V}\|_{2, \mu} \Big ) \tag{Cauchy-Schwarz inequality}\\
    & \leq \frac{2 \sqrt{C_2}}{1-\gamma} \| \hat{V}- \calC_\lambda^{\hat{\pi}} \hat{V}\|_{2, \mu}
\end{align*}
Therefore,
\begin{align*}
    J(\pi^\star) - J(\hat{\pi}) & = 
    J(\pi^\star) - J_\lambda(\pi^\star_\lambda)  + J_\lambda(\pi^\star_\lambda) - J_\lambda(\hat{\pi}) 
    + J_\lambda(\hat{\pi}) - J(\hat{\pi}) \\
    & = J_\lambda(\pi^\star) - J_\lambda(\pi^\star_\lambda)  + J_\lambda(\pi^\star_\lambda) - J_\lambda(\hat{\pi}) 
    + J_\lambda(\hat{\pi}) - J(\hat{\pi}) \tag{$\pi^\star$ is deterministic policy, $J(\pi^\star) = J_\lambda(\pi^\star)$}\\
    & \leq J_\lambda(\pi^\star_\lambda) - J_\lambda(\hat{\pi}) + \lambda \E \left [\sum_{t=0}^\infty \gamma^t \calH ( \hat{\pi}(\cdot \mid s_t)) \mid s_0 \sim d_0, \pi^\star \right] \tag{$J_\lambda(\pi^\star) - J_\lambda(\pi^\star_\lambda) \leq 0$ by optimality of $\pi^\star_\lambda$} \\
    & \leq \frac{2 \sqrt{C_2}}{1-\gamma} \| \hat{V}- \calC_\lambda^{\hat{\pi}} V\|_{2, \mu} + \frac{\lambda \ln |\A|}{1-\gamma}
\end{align*}

\end{proof}
\section{Proof of Theorem~\ref{theorem: performance}}
\label{sec: theorem}
We provide here a complete analysis of the SBEED algorithm. Recall: 
\begin{align*}
    \hat{V}, \hat{\pi} & = \argmin_{V \in  \calV, \pi \in \calP} \max_{g \in \calG}\el_{\calD}(V; V, \pi) - \er_{\calD}(g; V, \pi), \\
    \hat{g}_{V, \pi} & = \argmin_{g \in \calG} \er_{\calD}(g; V, \pi)
\end{align*}

\subsection{Dependence on the helper function class $\calG$}
Let $\bar{g}_{V, \pi} \triangleq \argmin_{g \in \calG} \E [\er_{\calD}(g; V, \pi)] = \argmin_{g \in \calG} \|g - \calC_\lambda^\pi V \|^2_{2, \mu}$ the best function in class $\calG$.
\begin{align*}
    \er_{\calD}(\hat{g}_{V, \pi}; V, \pi) -  \er_{\calD}(\calC_\lambda^\pi V; V, \pi) & = 
    \er_{\calD}(\hat{g}_{V, \pi}; V, \pi) - \er_{\calD}(\bar{g}_{V, \pi}; V, \pi)
    + \er_{\calD}(\bar{g}_{V, \pi}; V, \pi) 
    - \er_{\calD}(\calC_\lambda^\pi V; V, \pi) \\
    & = \frac{1}{n} \sum_{i=1}^n X_i(\hat{g}_{V, \pi}, V, \pi,\bar{g}_{V, \pi}) + \frac{1}{n} \sum_{i=1}^n Y_i(\bar{g}_{V, \pi}, V, \pi),
\end{align*}
where we define the following random variables
\begin{align*}
    & X(\hat{g}_{V, \pi}, V, \pi,\bar{g}_{V, \pi}) \triangleq 
    (\hat{g}_{V, \pi}(s, a) - r - \gamma V(s') + \lambda \ln \pi(a \mid s))^2 - (\bar{g}_{V, \pi}(s, a) - r - \gamma V(s') + \lambda \ln \pi(a \mid s))^2 \\
    & Y(\bar{g}_{V, \pi}, V, \pi) \triangleq (\bar{g}_{V, \pi}(s, a) - r - \gamma V(s') + \lambda \ln \pi(a \mid s))^2 - ((\calC_\lambda^\pi V)(s, a) - r - \gamma V(s') + \lambda \ln \pi(a \mid s))^2
\end{align*}
and for $i \in [n]$, $X_i(\hat{g}_{V, \pi}, V, \pi,\bar{g}_{V, \pi})$ and $Y_i(\bar{g}_{V, \pi}, V, \pi)$ is an i.i.d sample when $(s, a, r, s') = (s_i, a_i, r_i, s'_i)$.
\begin{lemma}[Properties of $X(\hat{g}_{V, \pi}, V, \pi,\bar{g}_{V, \pi})$] \label{lemma: X} We have 
\begin{compactenum}[\hspace{10pt} (i)]
    \item $| X(\hat{g}_{V, \pi}, V, \pi,\bar{g}_{V, \pi})| \leq 8 V_{\lambda, \max}^2 $
    \item $\E[X(\hat{g}_{V, \pi}, V, \pi,\bar{g}_{V, \pi})] = \| \hat{g}_{V, \pi} - \calC_\lambda^\pi V \|^2_{2,\mu} - \| \bar{g}_{V, \pi} - \calC_\lambda^\pi V \|^2_{2,\mu} \geq 0$
    \item $\var[X(\hat{g}_{V, \pi}, V, \pi,\bar{g}_{V, \pi})] \leq 32 V_{\lambda, \max}^2 \left( \E[X(\hat{g}_{V, \pi}, V, \pi,\bar{g}_{V, \pi})]  + 2 \epsilon_{\calG, \calV, \calP}\right) $
\end{compactenum}
\end{lemma}

\begin{proof}
For $(i)$, we have for any $(V,\pi, g) \in (\calV, \calP, \calG)$, $0 \leq r + \gamma V(s') - \lambda \ln \pi(a \mid s) \leq 2 V_{\lambda, \max}$ and $0 \leq g(s, a) \leq 2 V_{\lambda, \max}$ by definition of the functions classes, which implies 
$|g(s, a)- r - \gamma V(s') + \lambda \ln \pi(a \mid s)| \leq 2 V_{\lambda, \max}$. Therefore $| X(\hat{g}_{V, \pi}, V, \pi,\bar{g}_{V, \pi})| \leq 8 V_{\lambda, \max}^2 $.

For $(ii)$, we have
\begin{equation}
    \forall g \in \Real^{\S \times \A}, \| g - \calC_\lambda^\pi V \|^2_{2,\mu} = \E[\er_{\calD}(g; V, \pi)] - \E[\er_{\calD}(\calC_\lambda^\pi V; V, \pi)]
\end{equation}
Therefore, 
\begin{align*}
    \E[X(\hat{g}_{V, \pi}, V, \pi,\bar{g}_{V, \pi})] & = \E[\er_{\calD}(\hat{g}_{V, \pi}; V, \pi)] - \E[\er_{\calD}(\bar{g}_{V, \pi}; V, \pi)] \\
    & = \left( \E[\er_{\calD}(\hat{g}_{V, \pi}; V, \pi)] - \E[\er_{\calD}(\calC_\lambda^\pi V; V, \pi)] \right) - \left( \E[\er_{\calD}(\bar{g}_{V, \pi}; V, \pi)] - \E[\er_{\calD}(\calC_\lambda^\pi V; V, \pi)]\right) \\
    & = \| \hat{g}_{V, \pi} - \calC_\lambda^\pi V \|^2_{2,\mu} - \| \bar{g}_{V, \pi} - \calC_\lambda^\pi V \|^2_{2,\mu}
\end{align*}
and $\E[X(\hat{g}_{V, \pi}, V, \pi,\bar{g}_{V, \pi})] \geq 0$ by optimality of $\bar{g}_{V, \pi}$.

For $(iii)$, 
\begin{align*}
    & \var[X(\hat{g}_{V, \pi}, V, \pi,\bar{g}_{V, \pi})] \\ 
    & \leq \E[X(\hat{g}_{V, \pi}, V, \pi,\bar{g}_{V, \pi})^2] \\
    & = \E \Big[\Big( 
    (\hat{g}_{V, \pi}(s, a) - r - \gamma V(s') + \lambda \ln \pi(a \mid s))^2 - (\bar{g}_{V, \pi}(s, a) - r - \gamma V(s') + \lambda \ln \pi(a \mid s))^2 \Big)^2 \Big] \\
    & = \E \Big[ ( \hat{g}_{V, \pi}(s, a) - \bar{g}_{V, \pi}(s, a) )^2 (\hat{g}_{V, \pi}(s, a) +  \bar{g}_{V, \pi}(s, a) - 2 r - 2\gamma V(s') + 2\lambda \ln \pi(a \mid s))^2 \Big] \tag{$a^2 - b^2 = (a-b) (a+b)$}\\
    & \leq 16 V_{\lambda, \max}^2 \| \hat{g}_{V, \pi} - \bar{g}_{V, \pi}\|_{2,\mu}^2 \\
    & \leq 32 V_{\lambda, \max}^2 \left(\| \hat{g}_{V, \pi} - \calC_\lambda^\pi V \|_{2, \mu}^2 + \| \calC_\lambda^\pi V - \bar{g}_{V, \pi} \|_{2,\mu}^2\right) \tag{$(a +b)^2 \leq 2 a^2 + 2 b^2$}\\
    & = 32 V_{\lambda, \max}^2 \left( \E[X(\hat{g}_{V, \pi}, V, \pi,\bar{g}_{V, \pi})]  + 2 \| \calC_\lambda^\pi V - \bar{g}_{V, \pi} \|_{2,\mu}^2\right) \\
    & \leq 32 V_{\lambda, \max}^2 \left( \E[X(\hat{g}_{V, \pi}, V, \pi,\bar{g}_{V, \pi})]  + 2 \epsilon_{\calG, \calV, \calP}\right),
\end{align*}
\end{proof}

\begin{lemma}[Properties of $Y(\bar{g}_{V, \pi}, V, \pi)$] \label{lemma: Y} We have 
\begin{compactenum}[\hspace{10pt} (i)]
    \item $| Y(\bar{g}_{V, \pi}, V, \pi)| \leq 8 V_{\lambda, \max}^2 $
    \item $\E[Y(\bar{g}_{V, \pi}, V, \pi)] = \| \bar{g}_{V, \pi} - \calC_\lambda^\pi V \|^2_{2,\mu} $
    \item $\var[Y(\bar{g}_{V, \pi}, V, \pi)] \leq 16 V_{\lambda, \max}^2 \epsilon_{\calG, \calV, \calP} $

\end{compactenum}

\end{lemma}

\begin{proof}
$(i)$ and $(ii)$ can be checked similarly to what was done in Lemma~\ref{lemma: X}. For $(iii)$, we have
\begin{align*}
    & \var[Y(\bar{g}_{V, \pi}, V, \pi)]\\
    & \leq \E[Y(\bar{g}_{V, \pi}, V, \pi)^2] \\
    & = \E \left[ \left( (\bar{g}_{V, \pi}(s, a) - r - \gamma V(s') + \lambda \ln \pi(a\mid s))^2 - ((\calC_\lambda^\pi V)(s, a) - r - \gamma V(s') + \lambda \ln \pi(a \mid s))^2\right)^2 \right] \\
    & = \E \left[ (\bar{g}_{V, \pi}(s, a) - (\calC_\lambda^\pi V)(s, a))^2 \cdot (\bar{g}_{V, \pi}(s, a) +  (\calC_\lambda^\pi V)(s, a) -2 r - 2\gamma V(s') + 2\lambda \ln \pi(a\mid s))^2\right] \\
    & \leq 16 V_{\lambda, \max}^2 \| \bar{g}_{V, \pi} - \calC_\lambda^\pi V\|_{2, \mu}^2 \\
    & \leq 16 V_{\lambda, \max}^2 \epsilon_{\calG, \calV, \calP}
\end{align*}
\end{proof}

\subsection{Bound on $\frac{1}{n} \sum_{i=1}^n X_i(\hat{g}_{V, \pi}, V, \pi,\bar{g}_{V, \pi})$}

We apply Bernstein inequality (Lemma~\ref{lemma: bernstein}) and union bound over all $V, \pi, g \in \calV \times \calP \times \calG$. With probability at least $1 - \delta'$, we have

\begin{align} \label{eq: bernstein_1}
     \Big| \frac{1}{n} \sum_{i=1}^n X_i(\hat{g}_{V, \pi}, V, \pi,\bar{g}_{V, \pi})  - \E[X(\hat{g}_{V, \pi}, V, \pi,\bar{g}_{V, \pi})] \Big|
     & \leq \sqrt{\frac{2 \var[X(\hat{g}_{V, \pi}, V, \pi,\bar{g}_{V, \pi})] \ln \frac{|\calV| |\calP| |\calG|}{\delta'}}{n}} + \frac{16 V_{\lambda, \max}^2  \ln \frac{|\calV| |\calP| |\calG|}{\delta'}}{3n} 
\end{align}
We use the variance bound in $(iii)$ of Lemma~\ref{lemma: X} and we define the log factor $\imath \triangleq V_{\lambda, \max}^2  \ln \frac{|\calV| |\calP| |\calG|}{\delta'}$. Equation~\ref{eq: bernstein_1} becomes
\begin{align} \label{eq: bernstein_2}
     \Big| \frac{1}{n} \sum_{i=1}^n X_i(\hat{g}_{V, \pi}, V, \pi,\bar{g}_{V, \pi})  - \E[X(\hat{g}_{V, \pi}, V, \pi,\bar{g}_{V, \pi})] \Big|
     & \leq \sqrt{\frac{ 64 \imath \left( \E[X(\hat{g}_{V, \pi}, V, \pi,\bar{g}_{V, \pi})]  + 2 \epsilon_{\calG, \calV, \calP}\right)}{n}} + \frac{16 \imath}{3n} 
\end{align}
Since $\frac{1}{n} \sum_{i=1}^n X_i(\hat{g}_{V, \pi}, V, \pi,\bar{g}_{V, \pi}) = \er_{\calD}(\hat{g}_{V, \pi}; V, \pi) - \er_{\calD}(\bar{g}_{V, \pi}; V, \pi) \leq \er_{\calD}(\hat{g}_{V, \pi}; V, \pi) - \er_{\calD}(\hat{g}_{V, \pi}; V, \pi)= 0 $ by the optimality of $\hat{g}_{V, \pi}$, we obtain

\begin{align}
      \E[X(\hat{g}_{V, \pi}, V, \pi,\bar{g}_{V, \pi})] 
     & \leq \sqrt{\frac{ 64 \imath \left( \E[X(\hat{g}_{V, \pi}, V, \pi,\bar{g}_{V, \pi})]  + 2 \epsilon_{\calG, \calV, \calP}\right)}{n}} \quad + \frac{16 \imath}{3n} 
\end{align}

Using the fact $0 \leq x \leq \sqrt{ax + b} + c \Rightarrow 
    x \leq  a + \sqrt{a^2 + 2 (b + c^2)}$ (cf Lemma~\ref{lemma: quadratic inequality}) and by setting $a = \frac{64 \imath}{n}$, $b = \frac{128 \imath}{n} \epsilon_{\calG, \calV, \calP}$ and $c = \frac{16 \imath}{3n}$, we obtain 
    
\begin{align*}
\E[X(\hat{g}_{V, \pi}, V, \pi,\bar{g}_{V, \pi})] 
    & \leq \sqrt{\left( \frac{64 \imath}{n} \right)^2 + 2 \frac{128 \imath}{n} \epsilon_{\calG, \calV, \calP} + 2 \left( \frac{16 \imath}{3n} \right)^2}+ \frac{ 64\imath}{n} \\
    & \leq {\cal O} \left( \sqrt{\left(\frac{ \imath}{n}\right)^2 +  \frac{ \imath}{n} \epsilon_{\calG, \calV, \calP}} 
    + \frac{\imath}{n} \right) \\
    & \leq {\cal O} \left( \frac{\imath}{n} + \sqrt{\frac{ \imath}{n} \epsilon_{\calG, \calV, \calP}} + \frac{\imath}{n} \right) \tag{$\sqrt{a+b} \leq \sqrt{a} + \sqrt{b}$}\\
    & = {\cal O} \left( \frac{\imath}{n} + \sqrt{\frac{\imath}{n} \epsilon_{\calG, \calV, \calP}} \right)
\end{align*}

Substituting the above bound of $\E[X(\hat{g}_{V, \pi}, V, \pi,\bar{g}_{V, \pi})]$ in the inequality~\ref{eq: bernstein_2}, we obtain

\begin{align*}
     \Big| \frac{1}{n} \sum_{i=1}^n X_i(\hat{g}_{V, \pi}, V, \pi,\bar{g}_{V, \pi}) \Big | 
     & \leq {\cal O} \Big( \frac{\imath}{n} + \sqrt{\frac{\imath}{n} \epsilon_{\calG, \calV, \calP}} + \sqrt{\frac{ \imath \left( \frac{\imath}{n} + \sqrt{\frac{\imath}{n} \epsilon_{\calG, \calV, \calP}} + \epsilon_{\calG, \calV, \calP}\right)}{n}} + \frac{\imath}{n} \Big) \\
     & \leq {\cal O} \Big( \frac{\imath}{n} + \sqrt{\frac{\imath}{n} \epsilon_{\calG, \calV, \calP}} + \sqrt{ \left(\frac{\imath}{n}\right)^2 + \frac{i}{n} ( \frac{i}{n} + \epsilon_{\calG, \calV, \calP}) } \Big) \tag{$2 \sqrt{ab} \leq a + b$} \\ 
     & \leq {\cal O} \Big (  \frac{\imath}{n} + \sqrt{\frac{\imath}{n} \epsilon_{\calG, \calV, \calP}} \Big) \tag{$\sqrt{a + b} \leq \sqrt{a} + \sqrt{b} $}
\end{align*}

\subsection{Bound on $\frac{1}{n} \sum_{i=1}^n Y_i(\bar{g}_{V, \pi}, V, \pi)$}

We apply Bernstein inequality and union bound over $V, \pi, g \in \calV, \calP, \calG$, we have with probability at least $1-\delta'$
\begin{align*}
    \Big | \frac{1}{n} \sum_{i=1}^n Y_i(\bar{g}_{V, \pi}, V, \pi) - \E[Y(\bar{g}_{V, \pi}, V, \pi)]\Big | \leq \sqrt{\frac{2 \var[Y(\bar{g}_{V, \pi}, V, \pi)] \ln \frac{|\calV| |\calP| |\calG|}{\delta'}}{n}} + \frac{16 V_{\lambda, \max}^2 \ln \frac{|\calV| |\calP| |\calG|}{\delta'}}{3n}
\end{align*}
From $(ii)$ of Lemma~\ref{lemma: Y}, we have $\E[Y(\bar{g}_{V, \pi}, V, \pi)] = \| \bar{g}_{V, \pi} - \calC_\lambda^\pi V \|^2_{2,\mu} = \min_{g \in \calG} \| g - \calC_\lambda^\pi V \|^2_{2,\mu}$, then $\E[Y(\bar{g}_{V, \pi}, V, \pi)] \leq \epsilon_{\calG, \calV, \calP}$. Using the variance bound $(iii)$ of Lemma~\ref{lemma: Y}, we obtain 
\begin{align}
    \Big | \frac{1}{n} \sum_{i=1}^n Y_i(\bar{g}_{V, \pi}, V, \pi) \Big | & \leq \epsilon_{\calG, \calV, \calP} 
    + \sqrt{\frac{32 V_{\lambda, \max}^2 \epsilon_{\calG, \calV, \calP} \ln \frac{|\calV| |\calP| |\calG|}{\delta'}}{n}} + \frac{4 V_{\lambda, \max}^2 \ln \frac{|\calV| |\calP| |\calG|}{\delta'}}{3n} \nonumber \\
    & = {\cal O} \Big(\epsilon_{\calG, \calV, \calP} 
    + \sqrt{\frac{\imath}{n} \epsilon_{\calG, \calV, \calP}} + \frac{\imath}{n} \Big)
\end{align}

\subsection{Bound on $\er_{\calD}(\hat{g}_{V, \pi}; V, \pi) -  \er_{\calD}(\calC_\lambda^\pi V; V, \pi)$}
We have with probability at least $1- 2 \delta'$, for all $V, \pi \in \calV \times \calP$
\begin{align*}
    \Big | \er_{\calD}(\hat{g}_{V, \pi}; V, \pi) -  \er_{\calD}(\calC_\lambda^\pi V; V, \pi) \Big | 
    & = \Big | \frac{1}{n} \sum_{i=1}^n X_i(\hat{g}_{V, \pi}, V, \pi,\bar{g}_{V, \pi})  + \frac{1}{n} \sum_{i=1}^n Y_i(\bar{g}_{V, \pi}, V, \pi) \Big | \\
    & \leq {\cal O} \Big (  \frac{\imath}{n} + \sqrt{\frac{\imath}{n} \epsilon_{\calG, \calV, \calP}} \Big) + {\cal O} \Big(\epsilon_{\calG, \calV, \calP} 
    + \sqrt{\frac{\imath}{n} \epsilon_{\calG, \calV, \calP}} + \frac{\imath}{n} \Big) \\
    & =  {\cal O} \Big(\epsilon_{\calG, \calV, \calP} 
    + \sqrt{\frac{\imath}{n} \epsilon_{\calG, \calV, \calP}} + \frac{\imath}{n} \Big)
\end{align*}

\subsection{Bound on $\| \hat{V}- \calC_\lambda^{\hat{\pi}} \hat{V}\|_{2, \mu}$}
Denote $\bar{V}$ and $\bar{\pi}$ the best solution in the function class: $\bar{V}, \bar{\pi} \triangleq \argmin_{V \in \calV, \pi \calP} \| V- \calC_\lambda^{\pi} V\|_{2, \mu}^2$.

Define for all $V \in \calV$ and $\pi \in \calP$, the random variable: 
\begin{equation*}
    Z(V, \pi) \triangleq (V(s) - r - \gamma V(s') + \lambda \ln \pi(a \mid s))^2 - ((\calC_\lambda^\pi V)(s, a) - r- \gamma V(s') + \lambda \ln \pi(a \mid s))^2
\end{equation*}

\begin{lemma}[Properties of $Z(V, \pi)$] \label{lemma: Z} We have 
\begin{compactenum}[\hspace{10pt} (i)]
    \item $|  Z(V, \pi)| \leq 8 V_{\lambda, \max}^2 $
    \item $\E[Z(V, \pi)] = \| V - \calC_\lambda^\pi V \|^2_{2,\mu} $
    \item $\var[Z(V, \pi)] \leq 16 V_{\lambda, \max}^2 \epsilon_{\calG, \calV, \calP} $

\end{compactenum}

\end{lemma}
\begin{proof}
$(i)$ is obvious. For $(ii)$, we have
\begin{align*}
\E[Z(V, \pi)] & = \E \left[ \frac{1}{n} \sum_{i=1}^n Z_i(V, \pi) \right] =  \E [\el_{\calD}(V; V, \pi)] - \E [\er_{\calD}(\calC_\lambda^{\pi}; V, \pi)] = \| V- \calC_\lambda^{\pi} V\|_{2, \mu}^2
\end{align*}
For $(iii)$,
\begin{align*}
\var[Z(V, \pi)] & \leq \E[Z(V, \pi)^2] \\
    & \leq \E \left [ \left( (V(s) - r - \gamma V(s') + \lambda \ln \pi(a \mid s))^2 - ((\calC_\lambda^\pi V)(s, a) - r- \gamma V(s') + \lambda \ln \pi(a \mid s))^2 \right)^2\right]  \\
    & \leq \E \left [ (V(s) - (\calC_\lambda^\pi V)(s, a))^2 \cdot (V(s) + (\calC_\lambda^\pi V)(s, a) - 2r- 2\gamma V(s') + 2 \lambda \ln \pi(a \mid s))^2 \right] \\
    & \leq 16 V_{\lambda, \max}^2 \| V- \calC_\lambda^{\pi} V\|_{2, \mu}^2 \\
    & =  16 V_{\lambda, \max}^2 \E[Z(V, \pi)]
\end{align*}
\end{proof}

We would like to bound $\| \hat{V}- \calC_\lambda^{\hat{\pi}} \hat{V}\|_{2, \mu}^2 = \E[Z(\hat{V}, \hat{\pi})]$. We apply Bernstein inequality and union bound over all $V \in \calV$ and $\pi \in \calP$, we have with probability at least $1-\delta'$
\begin{align*}
    & \Big | \E[Z(\hat{V}, \hat{\pi})] - \frac{1}{n} \sum_{i=1}^n Z_i(\hat{V}, \hat{\pi})\Big | \leq  \sqrt{ \frac{2 \var[Z(\hat{V}, \hat{\pi})] \ln \frac{|\calV| |\calP|}{\delta'}}{n}} + \frac{16 V_{\lambda, \max}^2 \ln \frac{|\calV| |\calP|}{\delta'} }{3 n} \\
    & \Rightarrow \E[Z(\hat{V}, \hat{\pi})] \leq \frac{1}{n} \sum_{i=1}^n Z_i(\hat{V}, \hat{\pi}) + \sqrt{ \frac{32 \jmath \E[Z(\hat{V}, \hat{\pi})]}{n}} + \frac{4 \jmath }{3 n}, \tag{$(iii)$ of Lemma~\ref{lemma: Z}}
\end{align*}

where $\jmath \triangleq V_{\lambda, \max}^2 \ln \frac{|\calV| |\calP|}{\delta'}$. We need to bound $\frac{1}{n} \sum_{i=1}^n Z_i(\hat{V}, \hat{\pi})$.

\begin{align*}
    \frac{1}{n} \sum_{i=1}^n Z_i(\hat{V}, \hat{\pi}) & = \el_{\calD}(\hat{V}; \hat{V}, \hat{\pi}) - \er_{\calD}(\calC_\lambda^{\pi'}; \hat{V}, \hat{\pi}) \\
    & = \el_{\calD}(\hat{V}; \hat{V}, \hat{\pi}) - \er_{\calD}(\hat{g}_{\hat{V}, \hat{\pi}}; \hat{V}, \hat{\pi}) + 
    \er_{\calD}(\hat{g}_{\hat{V}, \hat{\pi}}; \hat{V}, \hat{\pi}) - \er_{\calD}(\calC_\lambda^{\hat{\pi}}; \hat{V}, \hat{\pi}) \\
    & \leq \el_{\calD}(\bar{V}; \bar{V}, \bar{\pi}) - \er_{\calD}(\hat{g}_{\bar{V}, \bar{\pi}}; \bar{V}, \bar{\pi}) + 
    \er_{\calD}(\hat{g}_{\hat{V}, \hat{\pi}}; \hat{V}, \hat{\pi}) - \er_{\calD}(\calC_\lambda^{\hat{\pi}}; \hat{V}, \hat{\pi}) \tag{optimality of $\hat{V}, \hat{\pi}$}\\
    & \leq \frac{1}{n} \sum_{i=1}^n Z_i(\bar{V}, \bar{\pi}) + \er_{\calD}(\calC_\lambda^{\bar{\pi}}; \bar{V}, \bar{\pi}) - \er_{\calD}(\hat{g}_{\bar{V}, \bar{\pi}}; \bar{V}, \bar{\pi}) + \er_{\calD}(\hat{g}_{\hat{V}, \hat{\pi}}; \hat{V}, \hat{\pi}) - \er_{\calD}(\calC_\lambda^{\hat{\pi}}; \hat{V}, \hat{\pi}) \\
    & \leq \frac{1}{n} \sum_{i=1}^n Z_i(\bar{V}, \bar{\pi}) + 2 \cdot {\cal O} \Big(\epsilon_{\calG, \calV, \calP} 
    + \sqrt{\frac{\imath}{n} \epsilon_{\calG, \calV, \calP}} + \frac{\imath}{n} \Big) \tag{with probability at least $1- 2 \delta'$}
\end{align*}
We have
\begin{align*}
   \E[Z(\bar{V}, \bar{\pi})] & = \| \bar{V} - \calC_\lambda^{\bar{\pi}} \bar{V}\|_{2, \mu}^2 = \min_{V \in \calV, \pi \in \calP} \| V - \calC_\lambda^{\pi} V\|_{2, \mu}^2 
   = \epsilon_{\calV, \calP} \\
    \var[Z(\bar{V}, \bar{\pi})] & \leq 16 V_{\lambda, \max}^2 \E[Z(\bar{V}, \bar{\pi})] = 16 V_{\lambda, \max}^2 \| \bar{V} - \calC_\lambda^{\bar{\pi}} \bar{V}\|_{2, \mu}^2
    = 16 V_{\lambda, \max}^2 \epsilon_{\calV, \calP}
\end{align*}

We have with probability $1 - \delta'$
\begin{align*}
     \frac{1}{n} \sum_{i=1}^n Z_i(\bar{V}, \bar{\pi}) &\leq \E[Z(\bar{V}, \bar{\pi})] + \sqrt{ \frac{2  \var[Z(\hat{V}, \hat{\pi})] \ln \frac{|\calV| |\calP|}{\delta'}}{n}} + \frac{16 V_{\lambda, \max}^2 \ln \frac{|\calV| |\calP|}{\delta'}}{3 n} \\
     & \leq \epsilon_{\calV, \calP} + \sqrt{ \frac{32 \jmath}{n} \epsilon_{\calV, \calP}} + \frac{4 \jmath }{3 n}
\end{align*}

this implies that with probability $1- 3 \delta '$
\begin{equation*}
    \frac{1}{n} \sum_{i=1}^n Z_i(\hat{V}, \hat{\pi}) = {\cal O} \Big ( \eta_1 + \eta_2 \Big),
\end{equation*}
where
\begin{align*}
    \eta_1 & = \epsilon_{\calV, \calP} + \sqrt{ \frac{\jmath}{n} \epsilon_{\calV, \calP}} + \frac{\jmath }{n} \\
    \eta_2 & = \epsilon_{\calG, \calV, \calP} 
    + \sqrt{\frac{\imath}{n} \epsilon_{\calG, \calV, \calP}} + \frac{\imath}{n}
\end{align*}
Therefore, with probabitily at least $1-4 \delta'$, we have

\begin{align*}
    \E[Z(\hat{V}, \hat{\pi})] &\leq {\cal O} \Big( \eta_1 + \eta_2 + \sqrt{ \frac{\jmath}{n}\E[Z(\hat{V}, \hat{\pi})]} + \frac{\jmath }{n} \Big) \\
    \Rightarrow  \E[Z(\hat{V}, \hat{\pi})] & \leq {\cal O} \Big( 
    \frac{\jmath}{n} + \sqrt{ \left( \frac{\jmath}{n} \right)^2 + \left( \frac{\jmath}{n} + \eta_1 + \eta_2 \right)^2}
    \Big) \\
   & \leq {\cal O} \Big( \frac{\jmath}{n} + \frac{\jmath}{n} + \frac{\jmath}{n} + \eta_1 + \eta_2 \Big) \\
   & \leq {\cal O} \Big( \frac{\jmath}{n} + \eta_1 + \eta_2 \Big) \\
\end{align*}
This implies that with probability $1-4 \delta'$, we have
\begin{align*}
    \| \hat{V} - \calC_\lambda^{\hat{\pi} \hat{V} }\|_{2, \mu} & \leq {\cal O} \Big ( \sqrt{\frac{j}{n} + \eta_1 + \eta_2} \Big) \\
    & \leq  {\cal O} \Big ( \sqrt{\frac{j}{n}} + \sqrt{\eta_1} + \sqrt{\eta_2} \Big) \\
    & \leq {\cal O} \Big ( \sqrt{\frac{j}{n}} + \sqrt{\epsilon_{\calV, \calP}} + \sqrt[4]{ \frac{\jmath}{n} \epsilon_{\calV, \calP}}+ \sqrt{\epsilon_{\calG, \calV, \calP}} + \sqrt[4]{\frac{\imath}{n} \epsilon_{\calG, \calV, \calP}} + \sqrt{\frac{\imath}{n}}\Big)
\end{align*}

\section{Technical Lemmas} \label{sec: technical lemmas}

\begin{lemma}[Bernstein inequality] \label{lemma: bernstein} let $X_1, \ldots, X_n$ be i.i.d and suppose $|X_i| \leq c$ and $\E[X_i] = \mu$. With probability at least $1-\delta$,
\begin{equation}
    \Big |\frac{1}{n} \sum_{i=1}^n X_i - \mu \Big | \leq \sqrt{\frac{2 \sigma^2 \ln (1/ \delta)}{n}} + \frac{2 c \ln(1/ \delta)}{3n},
\end{equation}
where $\sigma^2 = \frac{1}{n} \sum_{i=1}^n \var[X_i]$
\end{lemma}

\begin{lemma} ~\label{lemma: quadratic inequality} Let $a, b, c >0$, we have 
\begin{equation}
    0 \leq x \leq \sqrt{ax + b} + c \Rightarrow 
    x \leq  a + \sqrt{a^2 + 2 (b + c^2)}
\end{equation}
\end{lemma}

\begin{proof}
\begin{align*}
    0 \leq x \leq \sqrt{ax + b} + c & \Rightarrow 
    x^2 \leq (\sqrt{ax + b} + c)^2 \\
    & \Rightarrow  x^2 \leq 2 (ax + b) + 2 c^2 \tag{ $(a+b)^2 \leq 2a^2 + 2 b^2$}\\
    & \Rightarrow  0.5 x^2 - ax -( b - c^2) \leq 0
\end{align*}
the polynomial $P(x) = 0.5 x^2 - ax - (b +c^2)$ has two solutions 
$x_1 = a - \sqrt{a^2 + 2 (b + c^2)}$ and $x_2 = a + \sqrt{a^2 + 2 (b + c^2)}$. $P(x) \leq 0$ and $x \geq 0$ implies that $x \leq x_2$
\end{proof}
\end{document}